\documentclass[letterpaper,11pt]{article}
\usepackage{amssymb,amsbsy,amsmath,amscd,amsthm,amsfonts}
\usepackage{cite}
\usepackage{enumerate}
\usepackage{dsfont}
\usepackage{graphicx}
\usepackage{longtable}
\usepackage{tablefootnote}
\usepackage{standalone}
\usepackage{fancyhdr}
\usepackage{float}
\usepackage{caption}
\usepackage{subcaption}
\usepackage{url}
\usepackage[margin=1.25in]{geometry}

\RequirePackage[colorlinks = true,
            linkcolor = black,
            urlcolor  = black,
            citecolor = black,
            anchorcolor = black]{hyperref}

\newcommand{\Unif}{\mathrm{Uniform}}
\newcommand{\Laplace}{\mathrm{Laplace}}

\newcommand{\Expect}{\mathbb{E}}
\newcommand{\expect}[1]{\Expect\left[#1\right]}

\newcommand{\calA}{{\mathcal{A}}}
\newcommand{\calB}{{\mathcal{B}}}

\newcommand{\calD}{{\mathcal{D}}}

\newcommand{\calL}{{\mathcal{L}}}

\newcommand{\Th}{{^{\rm th}}}

\newtheorem{theorem}{Theorem}

\newtheorem{lemma}{Lemma}
\newtheorem{remark}{Remark}

\usepackage{prettyref}
\newrefformat{eq}{(\ref{#1})}
\newrefformat{thm}{Theorem~\ref{#1}}
\newrefformat{sec}{Section~\ref{#1}}
\newrefformat{algo}{Algorithm~\ref{#1}}
\newrefformat{fig}{Fig.~\ref{#1}}
\newrefformat{tab}{Table~\ref{#1}}
\newrefformat{rmk}{Remark~\ref{#1}}
\newrefformat{def}{Definition~\ref{#1}}
\newrefformat{cor}{Corollary~\ref{#1}}
\newrefformat{lmm}{Lemma~\ref{#1}}
\newrefformat{prop}{Proposition~\ref{#1}}
\newrefformat{app}{Appendix~\ref{#1}}
\newrefformat{ex}{Example~\ref{#1}}
\newrefformat{ass}{Assumption~\ref{#1}}

\usepackage{chngcntr}
\usepackage{apptools}
\AtAppendix{\counterwithin{lemma}{section}}
\AtAppendix{\counterwithin{equation}{section}}
\AtAppendix{\counterwithin{remark}{section}}

\providecommand{\keywords}[1]{\textbf{\text{Index terms ---}} #1}
\usepackage{authblk}

\begin{document}

\title{Estimating the Coefficients of a Mixture of Two Linear Regressions by Expectation Maximization}
\author[1]{Jason M. Klusowski\thanks{jason.klusowski@rutgers.edu}}
\author[2]{Dana Yang\thanks{xiaoqian.yang@yale.edu}}
\author[2]{W. D. Brinda\thanks{william.brinda@yale.edu}}
\affil[1]{Department of Statistics and Biostatistics, Rutgers University -- New Brunswick}
\affil[2]{Department of Statistics and Data Science, Yale University}

\maketitle

\begin{abstract}
We give convergence guarantees for estimating the coefficients of a symmetric mixture of two linear regressions by expectation maximization (EM). In particular, we show that the empirical EM iterates converge to the target parameter vector at the parametric rate, provided the algorithm is initialized in an unbounded cone. In particular, if the initial guess has a sufficiently large cosine angle with the target parameter vector, a sample-splitting version of the EM algorithm converges to the true coefficient vector with high probability. Interestingly, our analysis borrows from tools used in the problem of estimating the centers of a symmetric mixture of two Gaussians by EM.

We also show that the population EM operator for mixtures of two regressions is anti-contractive from the target parameter vector if the cosine angle between the input vector and the target parameter vector is too small, thereby establishing the necessity of our conic condition. Finally, we give empirical evidence supporting this theoretical observation, which suggests that the sample based EM algorithm performs poorly when initial guesses are drawn accordingly. Our simulation study also suggests that the EM algorithm performs well even under model misspecification (i.e., when the covariate and error distributions violate the model assumptions).
\end{abstract}

\keywords{Mixture models; expectation-maximization algorithm; iterative algorithms; clustering algorithms; regression analysis.}

\section{Introduction}\label{sec:introduction}

Mixtures of linear regressions are useful for modeling different linear relationships between input and response variables across several unobserved heterogeneous groups in a population. First proposed by \cite{quandt1978estimating} as a generalization of ``switching regressions'', this model has found broad applications in areas such as plant science \cite{turner2000estimating}, musical perception theory \cite{DeVeaux1989, Viele2002}, and educational policy \cite{ding2006using}.

In this paper, we consider estimating the model parameters in a symmetric two component mixture of linear regressions. Towards a theoretical understanding of this model, suppose we observe data $\calD_n = \{(X_i, Y_i)\}_{i=1}^n$, where
\begin{align} \label{eq:mixreg}
Y_i = R_i \langle\theta^*,X_i\rangle + \varepsilon_i,
\end{align}
$X_i \stackrel{i.i.d.}{\sim} N(0, I_d)$, $\varepsilon_i \stackrel{i.i.d.}{\sim} N(0, \sigma^2)$, $R_i \stackrel{i.i.d.}{\sim} \mbox{Rademacher}(1/2)$, and $ \{X_i\}, \{\varepsilon_i\} $, and $ \{R_i\}$ are independent of each other. In other words, each predictor variable is Gaussian, and the response is centered at either the $\theta^*$ or $-\theta^*$ linear combination of the predictor. The two classes are equally probable, and the label of each observation is unknown. We seek to estimate $\theta^*$ (or $-\theta^*$, which produces the same model distribution).

The likelihood function of the model
\begin{equation*}
\calL(\calD_n; \theta) = \prod_{i=1}^n\left[ \frac{1}{2}\psi(X_i)\psi_{\sigma}(Y_i - \langle\theta,X_i\rangle) + \frac{1}{2}\psi(X_i)\psi_{\sigma}(Y_i + \langle\theta,X_i\rangle)\right],
\end{equation*}
where $ \psi(x) = \frac{1}{(2\pi)^{d/2}}e^{-\|x\|^2/2} $ and $ \psi_{\sigma}(y) = \frac{1}{\sqrt{2\pi}\sigma}e^{-y^2/(2\sigma^2)} $,
is a multi-dimensional, multi-modal (it has many spurious local maxima), and nonconvex objective function, and hence direct maximization (e.g., grid search) is intractable. Even the population likelihood (in the infinite data setting) has global maxima at $ -\theta^* $ and $ \theta^* $, and a local minimum at the zero vector.  Given these computational concerns, other less expensive methods have been used to estimate the model coefficients. For example, mixtures of linear regressions can be interpreted as a particular instance of subspace clustering, since each regressor / regressand pair $ (X, Y) \in \mathbb{R}^{d+1} $ lies in the $ d $-dimensional subspace determined by their model parameter vectors ($ \theta^ * $ and $ -\theta^* $). When the covariates and errors are Gaussian, algebro-geometric and probabilistic interpretations of PCA \cite{vidal2005, tipping1999mixtures} motivate related clustering schemes, since there is an inherent geometric aspect to such mixture models. 

Another competitor is the Expectation-Maximization (EM) algorithm, which has been shown to have desirable empirical performance in various simulation studies \cite{DeVeaux1989}, \cite{Viele2002}, \cite{Faria2010}. Introduced in a seminal paper of Dempster, Laird, and Rubin \cite{Dempster1977}, the EM algorithm is a widely used technique for parameter estimation, with common applications in latent variable models (e.g., mixture models) and incomplete-data problems (e.g., corrupted or missing data) \cite{Beale1975}. It is an iterative procedure that monotonically increases the likelihood \cite[Theorem 1]{Dempster1977}. When the likelihood is not concave, it is well known that EM can converge to a non-global optimum \cite[page 97]{Wu1983}. However, recent work has side-stepped the question of whether EM reaches the likelihood maximizer, instead by directly working out statistical guarantees on its loss. For certain well-specified models, these explorations have identified regions of local contractivity of the EM operator near the true parameter so that, when initialized properly, the EM iterates approach the true parameter with high probability.

This line of research was spurred by \cite{Balakrishnan2014}, which established general conditions for which a ball centered at the true parameter would be a basin of attraction for the population version of the EM operator. For a large enough sample size, the difference (in that ball) between the sample EM operator and the population EM operator can be bounded such that the EM estimate approaches the true parameter with high probability. That bound is the sum of two terms with distinct interpretations. There is an \emph{algorithmic convergence} term $\gamma^t \| \theta^0 - \theta^* \|$ for initial guess $\theta^0$, truth $\theta^*$, and some modulus of contraction $\gamma \in (0, 1)$; this comes from the analysis of the population EM operator. The second term captures \emph{statistical convergence} and is proportional to the supremum norm of $ \sup_{\theta}\|M(\theta) - M_n(\theta)\| $, the difference between the population and sample EM operators, $ M $ and $ M_n $, respectively. This result is also shown for a ``sample-splitting" version of EM, where the sample is partitioned into batches and each batch governs a single step of the algorithm.

Our purpose here is to follow up on the analysis of \cite{Balakrishnan2014} by proving a larger basin of attraction for the mixture of two linear models and by establishing an exact probabilistic bound on the error of the sample-splitting EM estimate when the initial guess falls in the specified region. In particular, we show that

\begin{enumerate}[(a)]
\item The EM algorithm converges to the target parameter vector when it is initialized in a cone (defined in terms of the cosine similarity between the initial guess $ \theta^0 $ and the target model parameter $ \theta^* $).
\item The EM algorithm can fail to converge to $ \theta^* $ if the cosine similarity is too small.
\end{enumerate}

In related works, typically some variant of the mean value theorem is employed to establish contractivity toward the true parameter and the rate of geometric decay is then determined by relying heavily on the fact that initial guess belongs to a bounded set and is not too far from the target parameter vector (i.e., a ball centered at the target parameter vector). Our technique relies on Stein's Lemma, which allows us to reduce the problem to the two-dimensional case and exploit certain monotonicity properties of the population EM operator. Such methods allow one to be very careful and explicit in the analysis and more cleanly reveal the role of the initial conditions. These results cannot be deduced from preexisting works (such as \cite{Balakrishnan2014}), even by sharpening their analysis. Our improvements are not solely in terms of constants. Indeed, we will show that as long as the cosine angle between the initial guess and the target parameter vector (i.e., their degree of alignment) is sufficiently large, the EM algorithm converges to the target parameter vector $ \theta^* $. In particular, the norm of the initial guess can be arbitrarily large, provided the cosine angle condition is met.

In the machine learning community, mixtures of linear regressions are known as €œHierarchical Mixture of Experts€ (HME) and, there, the EM algorithm has also been employed \cite{Jordan1994}. The mixtures of linear regressions problem has also drawn recent attention from other scholars (e.g., \cite{Chen2014, chen2013convex, Yi2014, Chaganty2013, Zhong2016, sedghi2016provable, Yuanzhi2018}), although none of them have attempted to sharpen the EM algorithm in the sense that many works still require initialization is a small ball around the target parameter vector. For example, the general case with multiple components was considered in \cite{Zhong2016}, but initialization is still required to be in a ball around each of the true component coefficient vectors.

This paper is organized as follows. In Section~\ref{sec:population}, we explain the model and explain how the population EM operator is contractive toward the true parameter on a cone in $ \mathbb{R}^d $. We also show that the operator is not contractive toward the true parameter on certain regions of $ \mathbb{R}^d $. We connect our problem to phase retrieval in \prettyref{sec:phase} and borrow preexisting techniques to find a good initial guess in \prettyref{sec:initial}. Section~\ref{sec:sample} looks at the behavior of the sample-splitting EM operator in this cone and states our main result in the form of a high-probability bound. \prettyref{sec:proof} and \prettyref{sec:proofthm} are devoted to proving the contractivity of the population EM operator toward the target vector over a cone and proving our main result, respectively. A discussion of our findings, including  evidence of the failure of the EM algorithm for poor initial guesses from a simulated experiment, is provided in \prettyref{sec:discussion}. A simulation study of the EM algorithm under model misspecification is also given therein. Finally, more technical proofs are relegated to \prettyref{app:appendix}.

\section{The Empirical and Population EM Operator}\label{sec:population}

The EM operator for estimating $\theta^*$ (see \cite[page 6]{Balakrishnan2014} for a derivation) is
\begin{align}\label{eq:sample-em}
M_n(\theta) = \left(\frac{1}{n}\sum_{i=1}^n X_iX_i^{\top}\right)^{-1}\left[\frac{1}{n}\sum_{i=1}^n(2\phi(Y_i \langle \theta, X_i \rangle/\sigma^2)-1)X_iY_i\right],
\end{align}
where $\phi(z) = \frac{1}{1+e^{-2z}}$ is a horizontally stretched logistic sigmoid. Here $ \left(\frac{1}{n}\sum_{i=1}^n X_iX_i^{\top}\right)^{-1} $ is the inverse of the Gram matrix $ \frac{1}{n}\sum_{i=1}^n X_iX_i^{\top} $. In the limit with infinite data, the population EM operator replaces sample averages with expectations, and thus
\begin{align}\label{eq:population-em}
M(\theta) = 2\expect{\phi(Y \langle \theta, X \rangle/\sigma^2)XY}.
\end{align}


As we mentioned in the introduction, \cite{Balakrishnan2014} showed that if the EM operator \prettyref{eq:sample-em} is initialized in a ball around $ \theta^* $ with radius proportional $ \theta^* $, then the EM algorithm converges to $ \theta^* $ with high probability. It is natural to ask whether this good region of initialization can be expanded, possibly allowing for initial guesses with unbounded norm. The purpose of this paper is to relax the aforementioned ball condition of \cite{Balakrishnan2014} and show that if the cosine angle between $ \theta^* $ and the initial guess is not too small, the EM algorithm also converges. We also simplify the analysis considerably and use only elementary facts about multivariate Gaussian distributions.
Our improvement is manifested in the set containment
\begin{equation*}
\{ \theta: \|\theta-\theta^*\| \leq \sqrt{1-\rho^2}\|\theta^*\| \} \subseteq \{ \theta : \langle \theta, \theta^* \rangle \geq \rho\|\theta\|\|\theta^*\| \}, \quad \rho \in [-1,1],
\end{equation*}
since for all $\theta$ in the set on the left side,
\begin{align}
\langle\theta,\theta^*\rangle \nonumber
& = \frac{1}{2}\left(\|\theta\|^2+\|\theta^*\|^2-\|\theta-\theta^*\|^2\right) \nonumber \\
& \geq \frac{1}{2}\left(\|\theta\|^2+\rho^2\|\theta^*\|^2\right) \nonumber \\
& = \rho\|\theta\|\|\theta^\star\| + \frac{1}{2}(\rho\|\theta^*\|-\|\theta\|)^2 \nonumber \\
& \geq \rho\|\theta\|\|\theta^\star\|. \label{eq:innerprodbound}
\end{align}

The conditions in \cite[Corollary 5]{Balakrishnan2014} require the initial guess $\theta^0$ to be at most $\|\theta^\star\|/32$ away from $\theta^\star$, which corresponds to $ \|\theta\| \leq (1+\sqrt{1-\rho^2})\|\theta^*\| $ and $ \rho = \sqrt{1-(1/32)^2} \approx 0.999 $, whereas our condition allows for the norm of $ \theta $ to be unbounded and $ \rho > 0.85 $.

Let $ \theta_0 $ be the unit vector in the direction of $ \theta $ and let $ \theta^{\perp}_0 $ be the unit vector that belongs to the hyperplane spanned by $ \{ \theta^*, \theta \} $ and orthogonal to $ \theta $ (i.e., $ \theta^{\perp}_0 \in \text{span}\{\theta, \theta^* \} $ and $ \langle \theta, \theta_0^{\perp} \rangle = 0 $). Let $ \theta^{\perp} = \|\theta\|\theta^{\perp}_0 $. We will later show in \prettyref{sec:proof} that $M(\theta)$ belongs to $\text{span}\{\theta,\theta^\star\}$, as illustrated in \prettyref{fig:thetas}. Denote the angle between $\theta^*$ and $\theta_0$ as $\alpha$, with $ \|\theta^*\|\cos \alpha = \langle \theta_0, \theta^* \rangle$ and $ \rho = \cos\alpha $. As we will see from the following results, as long as $ \cos \alpha $ is not too small, $M(\theta)$ is a contracting operation that is always closer to the truth $\theta^*$ than $\theta$. The next lemma allows us to derive a region of $ \mathbb{R}^d $ on which $ M $ is contractive toward $ \theta^* $. We defer its proof until \prettyref{sec:proof}.
\begin{figure}
  \centering
  \includegraphics{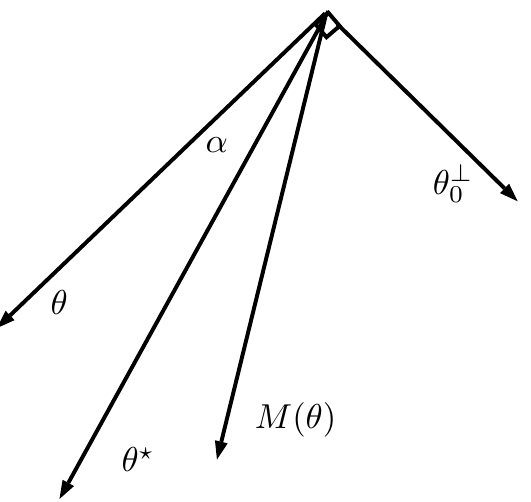}
  \caption{The population EM operator $ M(\theta) $ lies in the space spanned by $ \theta $ and $ \theta^* $. The unit vector $ \theta^{\perp}_0 $ lies in the space spanned by $ \theta $ and $ \theta^* $ and is perpendicular to $ \theta $. The vector $ \theta $ forms an angle $ \alpha $ with $ \theta^* $.}
\label{fig:thetas}
\end{figure}

\begin{lemma} \label{lmm:main}
For any $ \theta $ in $ \mathbb{R}^d $ with $ \langle \theta, \theta^* \rangle > 0 $,
\begin{equation} \label{eq:main_bound}
\|M(\theta)-\theta^*\| \leq \gamma\|\theta-\theta^*\|,
\end{equation}
where
\begin{equation} \label{eq:gamma}
\gamma = \sqrt{\kappa}\sqrt{1+4\left(\frac{|\langle \theta^{\perp}, \theta^* \rangle |+\sigma^2}{\langle \theta, \theta^* \rangle }\right)^2},
\end{equation}
and
\begin{equation} \label{eq:kappa}
\kappa^2 = \max\left\{1-\frac{|\langle\theta_0,\theta^\star\rangle|^2}{\sigma^2+\|\theta^*\|^2}, 1-\frac{\langle\theta,\theta^*\rangle}{\sigma^2+\langle\theta,\theta^*\rangle} \right\} < 1.
\end{equation}
\end{lemma}


If we define the input signal-to-noise ratio as $\eta'=\|\theta\|/\sigma$ and model signal-to-noise ratio (SNR) as $\eta=\|\theta^\star\|/\sigma$ and use the fact that $ \|\theta^\star\|\cos\alpha = \langle \theta_0,\theta^* \rangle$, then the contractivity constant \prettyref{eq:gamma} can be rewritten as
\begin{equation} \label{eq:cont-const}
\max\left\{ \left(1-\frac{\eta^2\cos^2\alpha}{1+\eta^2}\right)^{1/4}, \left(1- \frac{\eta'\eta\cos\alpha}{1+\eta'\eta\cos\alpha}\right)^{1/4} \right\}\sqrt{1+4\left(\tan\alpha+\frac{1}{\eta'\eta\cos\alpha}\right)^2}.
\end{equation}
 
\begin{remark} \label{rmk:contract}
If $ \eta' \geq 20 $, $ \eta \geq 40 $, and $ \cos\alpha\geq 0.85$, then $ \kappa $ is bounded by a universal constant less than $ 1/2 $ and $ \gamma $ is bounded by a universal constant less than $ 1 $, implying the population EM operator $\theta^{t} \leftarrow M(\theta^{t-1})$ converges to the truth $\theta^*$ exponentially fast.
\end{remark}

\section{Relationship to Phase Retrieval} \label{sec:phase}

The problem of estimating the true parameter vector in a mixture of two linear regressions is related to the phase retrieval problem, where one has access to magnitude-only data according to the model
\begin{equation} \label{eq:phase}
\widetilde Y = |\langle \theta^*, X \rangle|^2 + \varepsilon.
\end{equation}
In the no noise case, i.e., $ \varepsilon \equiv 0 $, one can obtain the phase retrieval model from the symmetric two component mixture of linear regressions by squaring each response variable $ Y_i $ from \prettyref{eq:mixreg} and visa versa by setting $ Y_i = R_i\sqrt{Y'_i} $, where $ R_i \stackrel{i.i.d.}{\sim} \text{Rademacher}(1/2) $ is independent of the data $ \{ (X_i, \widetilde Y_i) \}_{i=1}^n $. Here the sample subsets giving rise to the model parameters $ \theta^* $ and $ -\theta^* $ are $ \{i : R_i\text{sgn}(\langle \theta^*, X_i \rangle) = 1 \} $ and $ \{i : R_i\text{sgn}(\langle \theta^*, X_i \rangle) = -1 \} $, respectively. Even in the case of noise, squaring each response variable and subtracting the variance $ \sigma^2 $ of the error distribution yields
\begin{equation} \label{eq:connection}
Y'_i = Y^2_i - \sigma^2 = |\langle \theta^*, X_i \rangle|^2 + 2R_i\varepsilon_i\langle \theta^*, X_i \rangle + (\varepsilon^2_i-\sigma^2) = |\langle \theta^*, X_i \rangle|^2 + \xi(X_i, R_i, \varepsilon_i),
\end{equation}
where $ \xi(X_i, R_i, \varepsilon_i) $ is a mean zero random variable with variance $ 4\sigma^2\|\theta^*\|^2 + 2\sigma^4 $. This is essentially the phase retrieval model \prettyref{eq:phase} with heteroskedastic errors. See also \cite[Section 3.5]{chen2013convex} for a similar reduction to the ``Noisy Phase Model'', where the measurement error is pre-added to the inner product and then squared, viz., $  |\langle \theta^*, X \rangle + \varepsilon|^2 $.

Recent algorithms used to recover $ \theta^* $ from $ (X, \widetilde Y) $ include PhaseLift \cite{candes2013phaselift}, PhaseMax \cite{goldstein2018, dhifallah2017fundamental}, PhaseLamp \cite{dhifallah2018, dhifallah2017} and Wirtinger flow \cite{cai2016optimal, candes2015}, to name a few. PhaseLift operates by solving a semi-definite relaxation of the nonconvex formulation of the phase retrieval problem. PhaseMax and PhaseLamp solve a linear program over a polytope via convex programming. Finally, Wirtinger flow is an iterative gradient-based method that requires proper initialization.
Parallel to our work, \cite{dhifallah2018, dhifallah2017} reveal that exact recovery (when $ n, d \rightarrow +\infty $) in PhaseMax is governed by a critical threshold \cite[Theorem 3]{dhifallah2018}, which is measured in terms of the cosine angle between the initial guess and the target parameter vector. Analogous to our \prettyref{lmm:negative} (which is asymptotic in the sense that $ n \rightarrow +\infty $), they prove that recovery can fail is this cosine angle is too small. PhaseLamp is an iterative variant of PhaseMax that allows for a smaller cosine angle criterion than the critical threshold from PhaseMax. Our setting is slightly more general than \cite{dhifallah2018, dhifallah2017}  in that we allow for measurement error and our bounds are non-asymptotic in $ n $ and $ d $.

\section{Initialization} \label{sec:initial}

\prettyref{thm:main_splitting_empirical} below requires the initial guess to have a good inner product with $ \theta^* $. But how should one initialize in practice? There is considerable literature showing the efficacy of initialization based on spectral \cite{Yi2014}, \cite{Chaganty2013}, \cite{Zhong2016} or Bayesian \cite{Viele2002} methods. For example, inspired by the link \prettyref{eq:connection} between phase retrieval and our problem, we can use the same spectral initialization method of \cite{candes2015} for the Wirtinger flow iterates (c.f., \cite{Yi2014} for a similar strategy). That is, set
\begin{equation} \label{eq:lam}
\lambda^2 = d\frac{\sum_{i=1}^n Y'_i}{\sum_{i=1}^n \|X_i\|^2},
\end{equation}
and take $ \theta^0 $ equal to be the eigenvector corresponding to the largest eigenvalue of
\begin{equation} \label{eq:eigen}
\frac{1}{n}\sum_{i=1}^n Y'_i X_i X_i^{\top},
\end{equation}
scaled so that $ \|\theta^0\| = \lambda $. According to \cite[Theorem 3.3]{candes2015}, we are guaranteed that with high probability $ \|\theta^0-\theta^*\| \leq \frac{1}{8}\|\theta^*\| $, and hence by \prettyref{eq:innerprodbound}, $ \langle \theta^0, \theta^* \rangle \geq \sqrt{1-(1/8)^2}\|\theta^0\|\|\theta^*\| \approx 0.992\|\theta^0\|\|\theta^*\| $ and $ \|\theta^0\| \geq (7/8)\|\theta^*\| $. Provided that $ \|\theta^*\| \geq (8/7)20\sigma $, we will see in \prettyref{thm:main_splitting_empirical} that this $ \theta^0 $ satisfies our criteria for a good initial guess. Although the joint distributions of $ (X, \widetilde Y) $ and $ (X, Y') $ are not exactly the same, for large $ n $, $ \frac{1}{n}\sum_{i=1}^n\xi(X_i, R_i, \varepsilon_i) \approx 0 $, and hence \prettyref{eq:lam} and \prettyref{eq:eigen} are approximately equal to the same quantity with $ Y'_i $ replaced by $ \widetilde Y_i $. 

The next lemma, proved in \prettyref{app:appendix}, shows that the initialization conditions in \prettyref{rmk:contract} are essentially necessary in the sense that contractivity of $ M $ toward $ \theta^* $ can fail for certain initial guesses that do not meet our cosine angle criterion. In contrast, it is known \cite{daskalakis2016ten, Xu2016} that the population EM operator for a symmetric mixture of two Gaussians $ Y \sim \frac{1}{2}N(\theta^*, \sigma^2 I_d) + \frac{1}{2}N(-\theta^*, \sigma^2 I_d) $ is contractive toward $ \theta^* $ on the entire half-plane defined by $ \langle \theta, \theta^* \rangle > 0 $.\footnote{Note that this is the best one can hope for: if $ \langle \theta, \theta^* \rangle < 0 $ (reps. $ \langle \theta, \theta^* \rangle = 0 $), then the population EM operator is contractive toward $ -\theta^* $ (resp. the zero vector). Thus, unless $ \langle \theta, \theta^*\rangle = 0 $ (i.e., $ \theta $ belongs to the hyperplane perpendicular to $\theta^*$), the population EM is contractive towards either model parameter $ -\theta^* $ or $ \theta^* $.} The disparity between the EM operators for the two models is revealed in the proof of the contractivity of $ M $ toward $ \theta^* $ (see \prettyref{sec:proof}). Indeed, we will see in \prettyref{rmk:smoothed} that the population EM operator for mixtures of regressions is essentially a ``stretched'' version of the population EM operator for Gaussian mixtures.

\begin{lemma} \label{lmm:negative}
There is a subset of $ \mathbb{R}^d $ with positive Lebesgue measure, each of whose members $ \theta $ satisfies $ \langle \theta, \theta^* \rangle > 0 $ and
\begin{equation*}
\|M(\theta) - \theta^*\| > \|\theta-\theta^*\|.
\end{equation*}
\end{lemma}

While this result does not generally imply that the empirical iterates $  \theta^{t} \leftarrow M_n(\theta^{t-1}) $ will fail to converge to $ \theta^* $ for $  \langle \theta^0, \theta^* \rangle > 0 $, it does suggest that difficulties may arise in this regime. Indeed, the discussion in \prettyref{sec:discussion} gives empirical evidence for this theoretical observation.

\section{Main Theorem}\label{sec:sample}

As in \cite{Balakrishnan2014}, we analyze a sample-splitting version of the EM algorithm, where for an allocation of $ n $ samples and $ T $ iterations, we divide the data into $ T $ subsets of size $ \lfloor n/T \rfloor $. We then perform the updates $ \theta^{t} \leftarrow M_{n/T}(\theta^{t-1}) $, using a new subset of samples to compute $ M_{n/T}(\theta) $ at each iteration. The advantage of sample-splitting is purely for ease of analysis. In particular, conditional on the portion of data used to construct $ M_{n/T} $ at iteration $ t $, the distribution of $ \theta^{t} $ depends only on the other portion of the data through $ \theta^{t-1} $. For the next theorem, let $ \eta^0 = \|\theta^0\|/\sigma $ denote the initial SNR and $ \eta = \|\theta^*\|/\sigma $ denote the model SNR.

\begin{theorem} \label{thm:main_splitting_empirical}
Let $ \langle \theta^0, \theta^* \rangle > \rho\|\theta^0\|\|\theta^*\| $ for $ \rho > 0.85 $, $ \eta^0 \geq 20 $, and $ \eta \geq 40 $. Fix $ \delta \in (0, 1) $. Suppose furthermore that $ n \geq \max\{ cd\log(T/\delta), c' \} $ for some positive universal constant $ c $ and positive constant $ c' = c'(\rho, \sigma, \|\theta^*\|, \|\theta^0\|) $. Then there exists a universal modulus of contraction $ \gamma \in (0,1) $ and a positive universal constant $ C $ such that the sample-splitting empirical EM iterates $ (\theta^{t})_{t=1}^T $ based on $ n/T $ samples per step satisfy
\begin{equation*}
\|\theta^t-\theta^*\| \leq \gamma^t \|\theta^0 - \theta^* \| + \frac{C\sqrt{\sigma^2+\|\theta^*\|^2}}{1-\gamma}\sqrt{\frac{dT\log(T/\delta)}{n}},
\end{equation*}
with probability at least $ 1-\delta $.
\end{theorem}
Note that $ T $ governs the number of iterations of the EM operator; if it is too small, the term $ \gamma^t \|\theta^0 - \theta^* \| $ from \prettyref{thm:main_splitting_empirical} may fail to reach the parametric rate. Hence, $ T $ must scale like $ \frac{\log(n/d)}{\log(1/\gamma)} $.

We will prove \prettyref{thm:main_splitting_empirical} in \prettyref{sec:proofthm}. The main aspect of the analysis lies in showing that $ M_n $ satisfies an invariance property, i.e., $ M_n(\calA) \subseteq \calA $, where $ \calA $ is a set on which $ M $ is contractive toward $ \theta^* $. The algorithmic error $ \gamma^t \|\theta^0 - \theta^* \| $ is a result of repeated evaluation of the population EM operator $ \theta^t \leftarrow M(\theta^{t-1}) $ and the contractivity of $ M $ towards $ \theta^* $ from \prettyref{lmm:main}. The stochastic error $ \frac{C\sqrt{\sigma^2+\|\theta^*\|^2}}{1-\gamma}\sqrt{\frac{dT\log(T/\delta)}{n}} $ is obtained from a high-probability bound on $ \max_{t\in[T]}\|M_{n/T}(\theta^t)-M(\theta^t)\| $, which is contained in the proof of \cite[Corollary 5]{Balakrishnan2014}).

\section{Proof of \prettyref{lmm:main}}\label{sec:proof}

If $ W = \langle \theta^*, X \rangle + \varepsilon $, a few applications of Stein's Lemma~\cite[Lemma~1]{Stein1981} yields
\begin{align}
M(\theta) & = \expect{(2\phi(W \langle \theta, X \rangle/\sigma^2)-1)XW} \nonumber \\
& = \theta^*(\expect{2\phi(W \langle \theta, X \rangle/\sigma^2)+2(W\langle \theta, X \rangle /\sigma^2) \phi^{\prime}(W \langle \theta, X \rangle/\sigma^2)-1}) \nonumber \\ & \qquad + \theta\expect{2(W^2/\sigma^2)\phi^{\prime}(W \langle \theta, X \rangle/\sigma^2)}. \label{eq:span}
\end{align}
Letting
\begin{equation} \label{eq:A}
A = \expect{2\phi(W \langle \theta, X \rangle/\sigma^2)+2(W\langle \theta, X \rangle /\sigma^2) \phi^{\prime}(W \langle \theta, X \rangle/\sigma^2)-1},
\end{equation}
and
\begin{equation} \label{eq:B}
B = \expect{2(W^2/\sigma^2)\phi^{\prime}(W \langle \theta, X \rangle/\sigma^2)},
\end{equation}
we see that $ M(\theta) = \theta^*A + \theta B $ belongs to $ \text{span}\{\theta, \theta^* \} = \{ \lambda_1 \theta + \lambda_2 \theta^*, : \lambda_1, \lambda_2 \in \mathbb{R} \} $. This is a crucial fact that we will exploit multiple times.

Observe that for any $ a $ in $ \text{span}\{\theta, \theta^* \} $,
\begin{equation*}
a = \langle \theta_0, a \rangle \theta_0 + \langle \theta_0^{\perp}, a \rangle \theta_0^{\perp},
\end{equation*}
and
\begin{equation*}
\|a\|^2 = |\langle \theta_0, a \rangle|^2 + |\langle \theta_0^{\perp}, a \rangle |^2.
\end{equation*}
Specializing this to $ a = M(\theta)-\theta^* $ yields
\begin{equation*}
\|M(\theta)-\theta^*\|^2 = |\langle \theta_0, M(\theta)-\theta^* \rangle|^2 + |\langle \theta_0^{\perp}, M(\theta)-\theta^* \rangle |^2.
\end{equation*}
The strategy for establishing contractivity of $ M(\theta) $ toward $ \theta^* $ will be to show that the sum of $ |\langle \theta_0, M(\theta)-\theta^* \rangle|^2 $ and  $ |\langle \theta_0^{\perp}, M(\theta)-\theta^* \rangle|^2 $ is less than $ \gamma^2\|\theta-\theta^*\|^2 $. This idea was used in \cite{daskalakis2016ten} to obtain contractivity of the population EM operator for a mixture of two Gaussians. Due to the similarity of the two problems, it turns out that many of the same ideas transfer to our (more complicated) setting.

To reduce this $ (d+1) $-dimensional problem $ (X, Y) \in \mathbb{R}^{d+1} $ to a $ 2 $-dimensional problem $ (Z_1, Z_2) \in \mathbb{R}^2 $, we first show that
\begin{equation*}
W \langle \theta, X \rangle/\sigma^2 \stackrel{{\mathcal{D}}}{=} \Lambda Z_1|Z_2| + \Gamma Z^2_2,
\end{equation*}
where $ Z_1, Z_2 \stackrel{i.i.d.}{\sim} N(0,1) $. The coefficients $\Gamma$ and $\Lambda$ are
\begin{equation*}
\Gamma = \langle \theta, \theta^* \rangle/\sigma^2
\end{equation*}
and
\begin{equation*}
\Lambda^2 = (\|\theta\|^2/\sigma^4)(\sigma^2+\|\theta^*\|^2) - \Gamma^2 = (\|\theta\|^2/\sigma^4)(\sigma^2+|\langle \theta^{\perp}_0, \theta^* \rangle |^2).
\end{equation*}
This is because of the distributional equality
\begin{equation} \label{eq:dist1}
(W, \langle \theta, X \rangle/\sigma^2) \stackrel{{\mathcal{D}}}{=} \left(\sqrt{\sigma^2+\|\theta^*\|^2}Z_2,\; \frac{\Lambda}{\sqrt{\sigma^2+\|\theta^*\|^2}}Z_1+\frac{\Gamma}{\sqrt{\sigma^2+\|\theta^*\|^2}}Z_2\right).
\end{equation}
Note further that $\Lambda Z_1Z_2 + \Gamma Z^2_2  \stackrel{{\mathcal{D}}}{=} \Lambda Z_1|Z_2| + \Gamma Z^2_2 $ because they have the same moment generating function. Using this, we deduce that
\begin{equation} \label{eq:dist2}
W \langle \theta, X \rangle/\sigma^2 \stackrel{{\mathcal{D}}}{=} \Lambda Z_1|Z_2| + \Gamma Z^2_2.
\end{equation}

\prettyref{lmm:ABbounds} implies that
\begin{equation*}
(1-\kappa )\langle \theta^{\perp}_0, \theta^* \rangle \leq \langle \theta^{\perp}_0, M(\theta) \rangle \leq (1+\sqrt{\kappa})\langle \theta^{\perp}_0, \theta^* \rangle,
\end{equation*}
and consequently,
\begin{equation} \label{eq:perp_inequality}
|\langle \theta^{\perp}_0, M(\theta) - \theta^* \rangle| \leq \sqrt{\kappa}|\langle \theta^{\perp}_0, \theta -\theta^* \rangle| \leq \sqrt{\kappa}\|\theta-\theta^*\|.
\end{equation}

Next, we note that
\begin{align*}
\sigma^4|\Lambda^2-\Gamma| 
& = |\|\theta\|^2(\sigma^2+|\langle \theta^{\perp}_0, \theta^* \rangle |^2) - \sigma^2\langle \theta, \theta^* \rangle| \\
& \leq  \|\theta\|^2|\langle \theta^{\perp}_0, \theta^* \rangle |^2 + \sigma^2|\langle \theta, \theta-\theta^*\rangle| \\
& \leq \|\theta\|(|\langle \theta^{\perp}, \theta^* \rangle |+\sigma^2)\|\theta-\theta^*\|.
\end{align*}

Finally, define
\begin{equation*}
h(\alpha, \beta) = \expect{(2\phi(\alpha |Z_2|(Z_1 + \beta |Z_2|))-1)(|Z_2|(Z_1 + \beta |Z_2|))}.
\end{equation*}
Note that by definition of $ h $, $ h(\Lambda, \frac{\Gamma}{\Lambda}) = \frac{\langle \theta, M(\theta) \rangle}{\Lambda} $. In fact, $ h $ is the one-dimensional population EM operator for this model when $ \theta^* = \beta $ and $ \sigma^2 = 1 $.
By the self-consistency property of EM~\cite[page 79]{McLachlan2007}, $ h(\beta, \beta) = \beta $. Translating this to our problem, we have that $ h(\frac{\Gamma}{\Lambda}, \frac{\Gamma}{\Lambda}) = \frac{\Gamma}{\Lambda} = \frac{\langle \theta, \theta^* \rangle }{\sigma^2\Lambda} $. Since $ h(\Lambda, \frac{\Gamma}{\Lambda}) - h(\frac{\Gamma}{\Lambda}, \frac{\Gamma}{\Lambda}) = \int_{\frac{\Gamma}{\Lambda}}^{\Lambda}\frac{\partial h}{\partial \alpha} h(\alpha, \frac{\Gamma}{\Lambda})d\alpha $, we have from  \prettyref{lmm:derivative},
\begin{align*}
|\langle \theta_0, M(\theta)-\theta^* \rangle|
& \leq \frac{\sigma^2\Lambda}{\|\theta\|} \left|\int_{\frac{\Gamma}{\Lambda}}^{\Lambda}\frac{\partial}{\partial \alpha} h\left(\alpha, \frac{\Gamma}{\Lambda}\right)d\alpha\right| \\
& \leq \frac{2\sqrt{\kappa}\sigma^2\Lambda}{\|\theta\|} \left|\int_{\frac{\Gamma}{\Lambda}}^{\Lambda}\frac{d\alpha}{\alpha^2}\right| \\
& = \frac{2\sigma^2\sqrt{\kappa}|\Lambda^2-\Gamma|}{\Gamma\|\theta\|} \\
& \leq 2\sqrt{\kappa}\left(\frac{|\langle \theta^{\perp}, \theta^* \rangle |+\sigma^2}{\langle \theta, \theta^* \rangle }\right)\|\theta - \theta^*\|.
\end{align*}

Combining this with inequality \prettyref{eq:perp_inequality} yields \prettyref{eq:main_bound}. This completes the proof of \prettyref{lmm:main}.

\begin{remark} \label{rmk:smoothed}
The function $ h $ is related to the EM operator for the one-dimensional symmetric mixture of two Gaussians model $ Y \sim \frac{1}{2}N(-\beta, 1) + \frac{1}{2}N(\beta, 1) $. One can derive that (see~\cite[page 11]{Klusowski2016}) the population EM operator is 
\begin{equation*}
g(\alpha, \beta) = \expect{(2\phi(\alpha(Z_1+\beta))-1)(Z_1+\beta)}.
\end{equation*}
Then $ h(\alpha, \beta) $ is a ``stretched'' version of $ g(\alpha, \beta) $ as seen through the identity
\begin{equation*}
h(\alpha, \beta) = \expect{|Z_2|g(\alpha|Z_2|, \beta|Z_2|)}.
\end{equation*}

In light of this relationship, it is perhaps not surprising that the EM operator for the mixture of linear regressions problem also enjoys a large basin of attraction. 


On the other hand, from \cite[page 11]{Klusowski2016}, the population EM operator $ \widetilde M $ for the symmetric two component mixture of Gaussians $ Y \sim \frac{1}{2}N(\theta^*, \sigma^2 I_d) + \frac{1}{2}N(-\theta^*, \sigma^2 I_d) $, is equal to
\begin{equation*}
\widetilde M(\theta) = 2\expect{Y\phi(\langle Y, \theta \rangle/\sigma^2)} = \theta^* \widetilde A + \theta \widetilde B,
\end{equation*}
where $ \widetilde A = \expect{2\phi(\langle \theta, \theta^* \rangle/\sigma^2 + \|\theta\|Z_1/\sigma)-1} $ and $ \widetilde B = 2\expect{\phi'(\langle \theta, \theta^* \rangle/\sigma^2 + \|\theta\|Z_1/\sigma)} $. 


Compare the values of $ \widetilde A $ and $ \widetilde B $ with $ A $ and $ B $ from \prettyref{eq:A} and \prettyref{eq:B}. We see that $ M $ is essentially a ``stretched'' and ``scaled'' version of $ \widetilde M $ by the random dilation factors $ |Z_2|\sqrt{1+|\langle \theta^{\perp}_0, \theta^*\rangle|^2/\sigma^2} $ and $ |Z_2|\sqrt{1+\|\theta^*\|^2/\sigma^2} $. As will be seen in the proof \prettyref{lmm:negative} in \prettyref{app:appendix}, this additional source of variability causes the repellant behavior of $ M $ in \prettyref{lmm:negative}.
\end{remark}

\begin{remark}
Recently in \cite{brutzkus2017globally}, the authors analyzed gradient descent for a single-hidden layer convolutional neural network structure with no overlap and Gaussian input. In this setup, we observe {\it i.i.d.} data $ \{(X_i, Y_i)\}_{i=1}^n $, where $ Y_i = f(X_i,w) + \varepsilon_i $ and $ X_i \sim N(0, I_d) $ and $ \varepsilon_i \sim N(0, \sigma^2) $ are independent of each other. The neural network has the form $ f(x,w) = \frac{1}{k}\sum_{j=1}^k\max\{0, \langle w_j, x \rangle \} $ and the only nonzero coordinates of $ w_j $ are in the $ j^{\Th} $ successive block of $ d/k $ coordinates and are equal to a fixed $ d/k $ dimensional filter vector $ w $. One desires to minimize the risk $ \ell(w) = \expect{(f(X,w)-f(X,w^{\star}))^2} $. Interestingly, the gradient of $ \ell(w) $ belongs to the linear span of $ \omega $ and $ \omega^{\star} $, akin to our $ M(\theta) \in \text{span}\{\theta, \theta^* \} $ (and also in the Gaussian mixture problem \cite{Klusowski2016}). This property also plays a critical role in their analysis.

\end{remark}

\section{Proof of \prettyref{thm:main_splitting_empirical}}\label{sec:proofthm}

The first step of the proof is to show that the empirical EM operator satisfies $ M_n(\calA) \subset \calA $, where $ \calA $ is a set on which $ M $ is contractive toward $ \theta^* $. In other words, the empirical EM iterates remain in a set where $ M(\theta) $ is closer to $ \theta^* $ than its input $ \theta $. To this end, define the set $ \calA = \{ \theta: \langle \theta, \theta^* \rangle > \rho\|\theta\|\|\theta^*\|, \|\theta\| \geq 20\sigma \} $. By \prettyref{rmk:contract}, the stated conditions on $ \rho $, $ \|\theta\| $, and $ \|\theta^*\| $ ensure that $ M $ is contractive toward $ \theta^* $ on $ \calA $ and that $ \kappa < 1/2 $.

Next, we use  \prettyref{lmm:Mbounds} which shows that 
\begin{equation*} 
M(\calA) \subseteq \calB := \{ \theta: \langle \theta, \theta^* \rangle > (1+\Delta)\rho\|\theta\|\|\theta^*\|, \; \|\theta^*\|(1-\kappa) \leq \|\theta\| \leq \sqrt{\sigma^2 + 3\|\theta^*\|^2} \}.
\end{equation*}
The fact that $ \calB \subset \calA $ allows us to claim that when $ n $ is large enough, $ M_n(\calA) \subset M(\calB) $, and hence $ M_n(\calA) \subseteq M(\calA) \subseteq \calA $.
To show this, assume $ \sup_{\theta\in\calA}\|M_n(\theta) - M(\theta) \| < \epsilon $. That implies
\begin{equation}
\sup_{\theta \in \calA}\left\| \frac{M_n(\theta)}{\|M_n(\theta)\|} - \frac{M(\theta)}{\|M(\theta)\|} \right\| \leq 2\sup_{\theta \in \calA}\frac{\|M_n(\theta)-M(\theta)\|}{\|M(\theta)\|} < \frac{2\epsilon}{(1-\kappa)\|\theta^*\|}. \label{eq:close}
\end{equation}
For the last inequality, we used the fact that $ \|M(\theta)\| \geq \|\theta^*\|A \geq \|\theta^*\|(1-\kappa) $ for all $ \theta $ in $ \calA $, which follows from \prettyref{eq:span} and \prettyref{lmm:ABbounds}.
By \prettyref{eq:close} and \prettyref{lmm:Mbounds} \prettyref{eq:cosinea}, we have that
\begin{align*}
\sup_{\theta \in \calA} \left\langle \theta^*, \frac{M_n(\theta)}{\|M_n(\theta)\|} \right\rangle
& \geq \sup_{\theta \in \calA} \left\langle \theta^*, \frac{M(\theta)}{\|M(\theta)\|} \right\rangle - \frac{2\epsilon}{(1-\kappa)} \\
& \geq \|\theta^*\|(1+\Delta)\rho - \frac{2\epsilon}{(1-\kappa)} \\
& \geq \|\theta^*\|\rho,
\end{align*}
provided $ \epsilon < (\frac{1-\kappa}{2})\Delta \rho \|\theta^*\| $ and, by \prettyref{eq:span} and \prettyref{lmm:ABbounds},
\begin{align*}
\sup_{\theta \in \calA} \|M_n(\theta)\|
& \geq \sup_{\theta \in \calA} \| M(\theta) \| - \epsilon \\
& \geq \|\theta^*\|(1-\kappa) - \epsilon \\
& \geq 40\sigma(1-\kappa) - \epsilon \\
& \geq 20\sigma,
\end{align*}
provided $ \epsilon < 20\sigma(1-2\kappa) $, which is positive since $ \kappa < 1/2 $.

For $ \delta \in (0, 1) $, let $ \epsilon_M(n,\delta) $ be the smallest number such that for any fixed $ \theta $ in $ \calA $, we have
\begin{equation*}
\|M_n(\theta) - M(\theta) \| \leq \epsilon_M(n,\delta),
\end{equation*}
with probability at least $ 1-\delta $. Moreover, suppose $ c' = c'(\rho, \sigma, \|\theta^*\|, \|\theta^0\|) $ is a constant so that if $ n \geq c'  $, then
\begin{equation*}
\epsilon_M(n,\delta) \leq \min\left\{ 20\sigma(1-2\kappa), \left(\frac{1-\kappa}{2}\right)\Delta \rho \|\theta^*\| \right\}.
\end{equation*}
This guarantees that $ M_n(\calA) \subseteq \calA $.
For any iteration $ t \in [T] $, we have
\begin{equation*}
\|M_{n/T}(\theta^t) - M(\theta^t) \| \leq \epsilon_M(n/T,\delta/T),
\end{equation*}
with probability at least $ 1-\delta/T $. Thus by a union bound and $ M_n(\calA) \subseteq \calA $,
\begin{equation*}
\max_{t\in  [T] }\|M_{n/T}(\theta^t) - M(\theta^t) \| \leq \epsilon_M(n/T,\delta/T),
\end{equation*}
with probability at least $ 1 - \delta $.

Hence if $ \theta^0 $ belongs to $ \calA $, then by \prettyref{lmm:main},
\begin{align*}
\|\theta^{t} - \theta^* \|
& = \|M_{n/T}(\theta^{t-1}) - \theta^* \| \\
& \leq \|M(\theta^{t-1}) - \theta^* \| + \|M_{n/T}(\theta^t) - M(\theta^t) \| \\
& \leq \gamma \| \theta^{t-1} - \theta^*\| + \max_{t\in [T]}\|M_{n/T}(\theta) - M(\theta) \| \\
& \leq \gamma \| \theta^{t-1} - \theta^*\| + \epsilon_M(n/T,\delta/T),
\end{align*}
with probability at least $ 1 - \delta $.
Solving this recursive inequality yields,
\begin{align*}
\|\theta^{t} - \theta^* \|
& \leq \gamma^t\|\theta^0-\theta^*\| + \epsilon_M(n/T,\delta/T)\sum_{j=0}^{t-1}\gamma^j \\
& \leq \gamma^t\|\theta^0-\theta^*\| + \frac{\epsilon_M(n/T,\delta/T)}{1-\gamma},
\end{align*}
with probability at least $ 1 - \delta $.

Finally, by a slight modification to the proof of \cite[Corollary 5]{Balakrishnan2014} that uses $ M(\theta) \leq \sqrt{\sigma^2+3\|\theta^*\|^2} $ from \prettyref{eq:cosinea1}, it follows that if $ n \geq cd\log(T/\delta) $, then there exists a universal constant $ C > 0 $ such that
\begin{equation*}
\epsilon_M(n/T,\delta/T) \leq C\sqrt{\sigma^2+\|\theta^*\|^2}\sqrt{\frac{dT\log(T/\delta)}{n}}
\end{equation*}
with probability at least $ 1-\delta/T $. This completes the proof of \prettyref{thm:main_splitting_empirical}.

\section{Discussion} \label{sec:discussion}
In this paper, we showed that the empirical EM iterates converge to true coefficients of a mixture of two linear regressions as long as the initial guess lies within a cone (see the condition on \prettyref{thm:main_splitting_empirical}: $\langle\theta^0,\theta^\star\rangle>\rho\|\theta^0\|\|\theta^\star\|$).

In \prettyref{fig:SimEM}, we perform a simulation study of $  \theta^{t} \leftarrow M_n(\theta^{t-1}) $ with $ \sigma = 1 $, $ n = 1000 $, $ d = 2 $, and $ \theta^* = (-7/25, 24/25)^{\top} $. All entries of the covariate vector $X$ and the noise $\varepsilon$ are generated {\it i.i.d.} from a standard Gaussian distribution. We consider the error $ \|\theta^t - \theta^* \| $ plotted as a function of $ \cos \alpha = \frac{\langle \theta^0, \theta^* \rangle}{\|\theta^0\| \|\theta^*\|} $ at iterations $ t = 5, 10, 15, 20, 25 $ (darker lines correspond to larger values of $ t $). For each $ t $, we choose a unit vector $ \theta^0 $ so that $ \cos \alpha $ ranges between $ -1 $ and $ + 1 $.
In accordance with the theory we have developed, increasing the iteration size and increasing the cosine angle decreases the overall error. According to \prettyref{lmm:negative}, the algorithm should suffer when $ \cos \alpha $ is small. Indeed, we observe a sharp transition at $ \cos \alpha \approx 0.2 $. The algorithm converges to the other model parameter $ -\theta^* = (7/25, -24/25)^{\top} $ for initial guesses with cosine angle (approximately) smaller than $ 0.2 $. The plot in \prettyref{fig:SimEM2} is a zoomed-in version of \prettyref{fig:SimEM} near this transition point.

One of the shortcomings of the EM algorithm is that it is \emph{model dependent}, that is, the form of the EM operator is derived from the assumption of Gaussian input $ X $, error $ \varepsilon $, and two component assumption. It is natural to ask how changing either distribution and using the original EM operator designed for Gaussian data performs on simulated data. As a simple illustration, the simulation results in \prettyref{fig:SimEMmis} use $ X \sim \Unif([-\sqrt{3}, \sqrt{3}]^d) $ and $ \varepsilon \sim \Unif([-\sigma\sqrt{3}, \sigma\sqrt{3}]^d) $ (\prettyref{fig:uniform}) and $ X \sim N(0, I_d) $ and $ \varepsilon \sim \Laplace(0, \sigma/\sqrt{2}) $ (\prettyref{fig:laplace}) for $ \sigma^2 = 1 $. The performance is similar to \prettyref{fig:SimEM} and \prettyref{fig:SimEM2}, although note that in \prettyref{fig:uniform}, a larger cosine angle is required for convergence (i.e., cosine angles at least $ \cos \alpha \approx 0.4 $).

More generally, future work would rigorously study the effect of EM under model misspecification. In this direction, the recent work of \cite{dwivedi2018singularity} has analyzed the EM algorithm for over-fitted mixtures.
\begin{figure}[H]
\centering
\includegraphics[width = 0.5\textwidth]{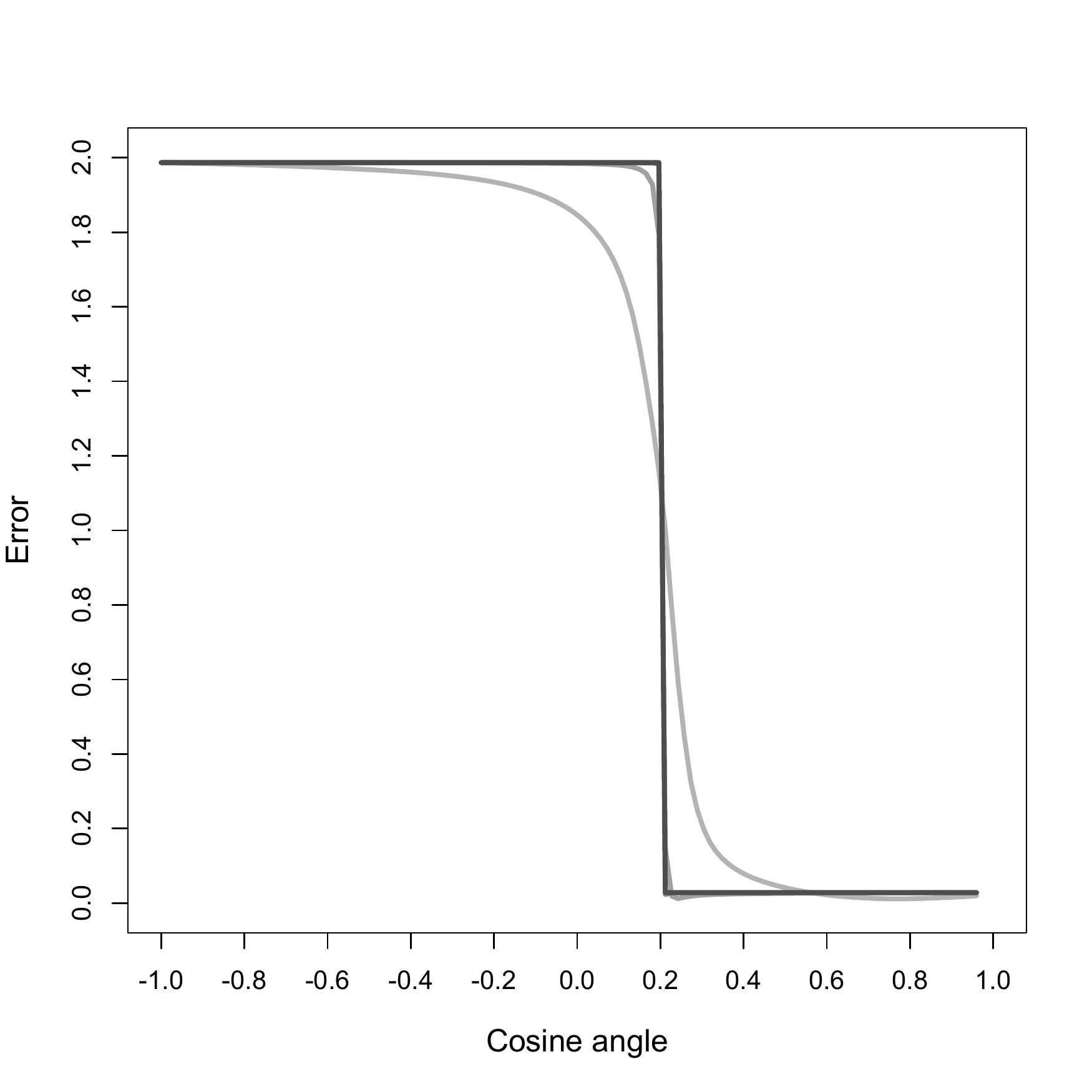}
\caption{A simulation study of $  \theta^{t} \leftarrow M_n(\theta^{t-1}) $ with $ \sigma^2 = 1 $, $ n = 1000 $, $ d = 2 $, and $ \theta^* = (-7/25, 24/25)^{\top} $. The values of $ t $ range from $ 5 $ to $ 25 $. The vertical axis is the error $ \|\theta^t - \theta^* \| $ and the horizontal axis is the cosine angle between the initial guess $ \theta^0 $ and $ \theta^* $. Darker lines correspond to larger values of $ t $.}
\label{fig:SimEM}
\end{figure}
\begin{figure}[H]
\centering
\includegraphics[width = 0.5\textwidth]{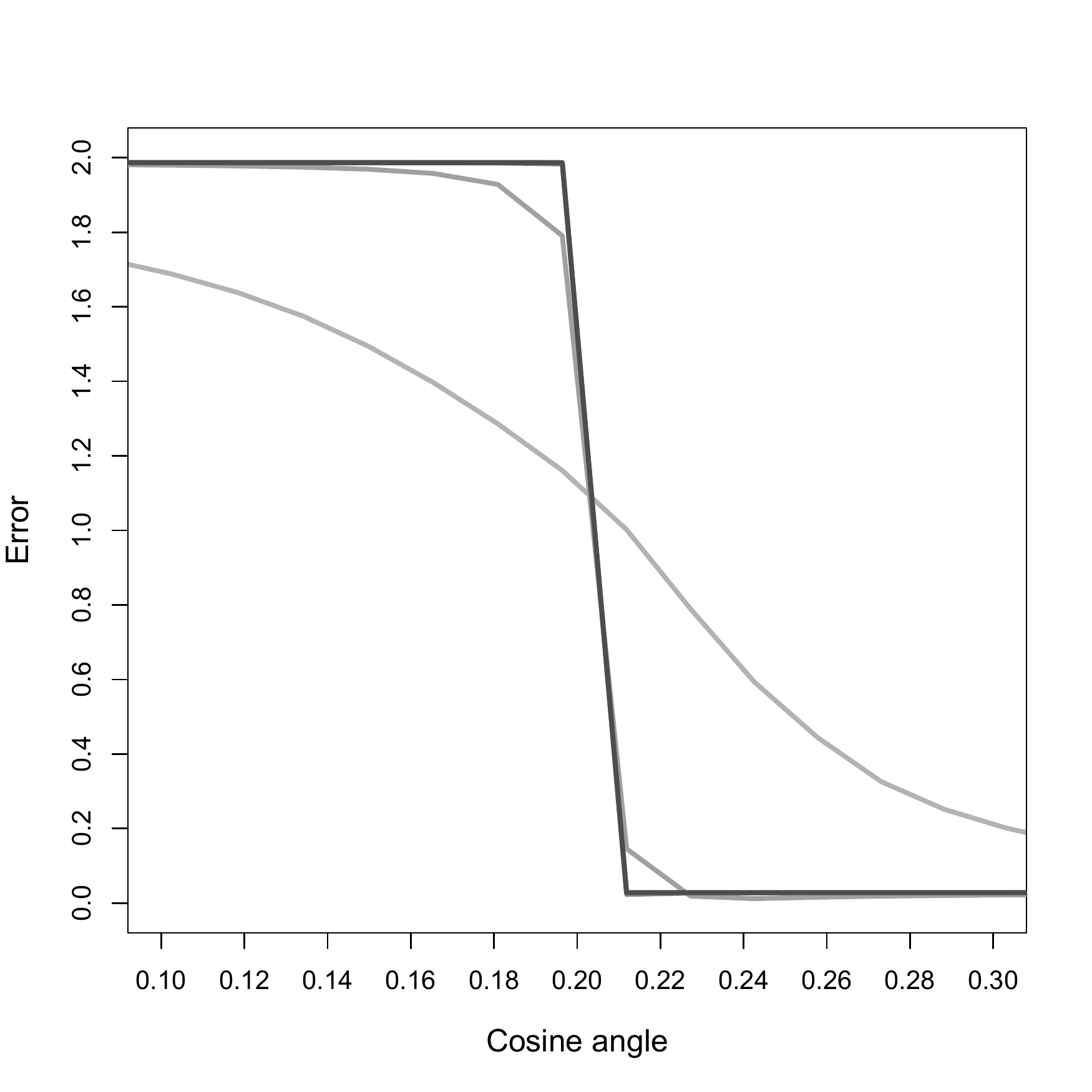}
\caption{A zoomed-in version of \prettyref{fig:SimEM} showing the transition point at $ \cos \alpha \approx 0.2 $.}
\label{fig:SimEM2}
\end{figure}

\begin{figure} [H]
\centering
\begin{minipage}[t]{0.4\textwidth}
  \centering
  \includegraphics[width=1\linewidth]{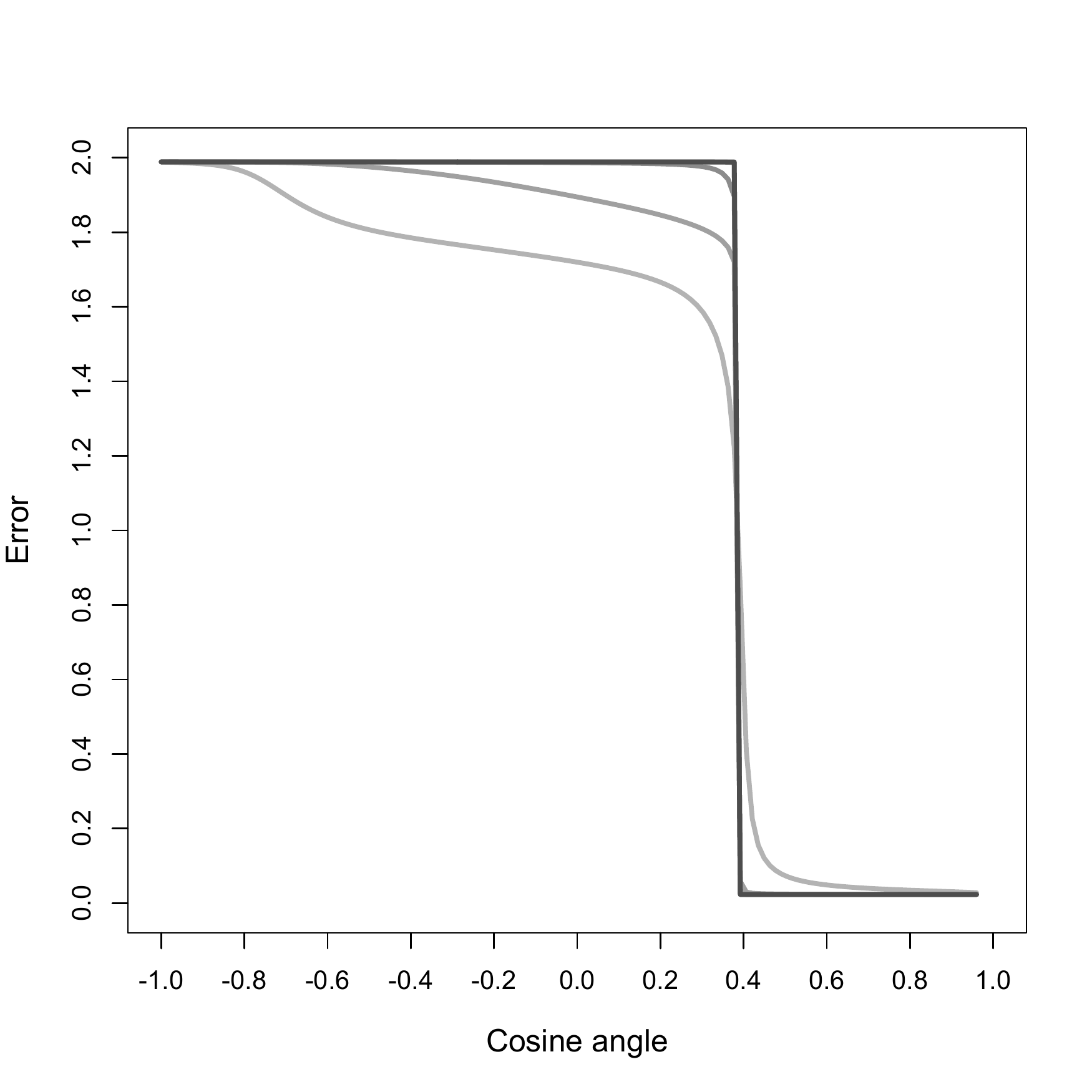}
\captionof{figure}{$ X \sim \Unif([-\sqrt{3}, \sqrt{3}]^d) $ and $ \varepsilon \sim \Unif([-\sqrt{3}, \sqrt{3}]^d) $}
  \label{fig:uniform}
\end{minipage}%
\qquad
\begin{minipage}[t]{0.4\textwidth}
  \centering
  \includegraphics[width=1\linewidth]{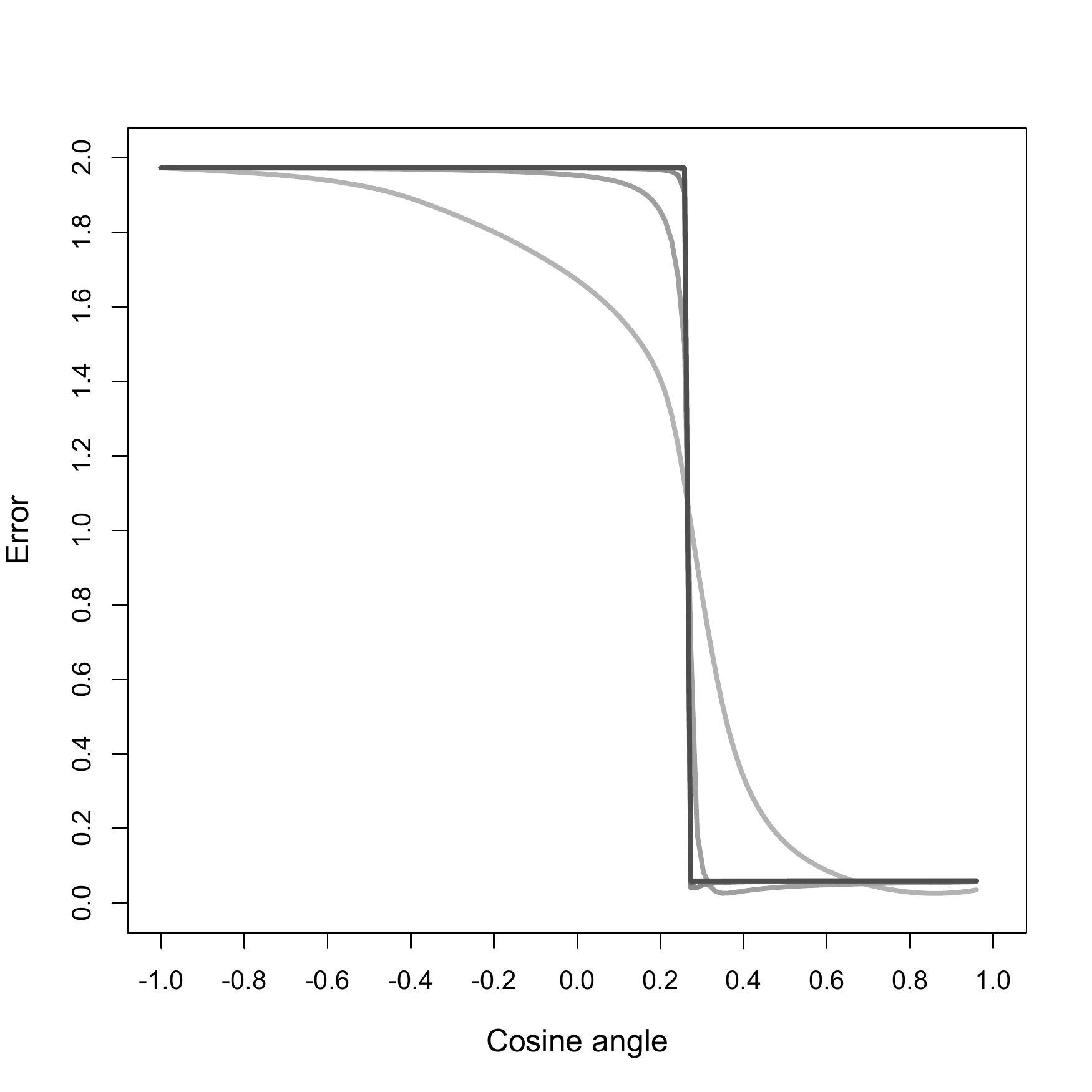}
\captionof{figure}{$ X \sim N(0, I_d) $ and $ \varepsilon \sim \Laplace(0, \sigma/\sqrt{2}) $}
  \label{fig:laplace}
\end{minipage}
\caption{A simulation study of $  \theta^{t} \leftarrow M_n(\theta^{t-1}) $ under model misspecification with $ \sigma^2 = 1 $, $ n = 1000 $, $ d = 2 $, $ \theta^* = (-7/25, 24/25)^{\top} $. The values of $ t $ range from $ 5 $ to $ 25 $. The vertical axis is the error $ \|\theta^t - \theta^* \| $ and the horizontal axis is the cosine angle between the initial guess $ \theta^0 $ and $ \theta^* $. Darker lines correspond to larger values of $ t $.}
\label{fig:SimEMmis}
\end{figure} 

\appendix

\section{Appendix} \label{app:appendix}
In this appendix, we prove \prettyref{lmm:negative} and all other supporting lemmas used in the body of the paper.

\begin{proof}[Proof of \prettyref{lmm:negative}]
Recall that in general, $ M(\theta) = \theta^* A + \theta B $, where
\begin{equation*}
A = \expect{2\phi(W \langle \theta, X \rangle/\sigma^2) + 2(W\langle \theta, X \rangle/\sigma^2) \phi^{\prime}(W \langle \theta, X \rangle/\sigma^2)-1},
\end{equation*}
\begin{equation*}
B = 2\expect{(W^2/\sigma^2)\phi^{\prime}(W \langle \theta, X \rangle/\sigma^2)}.
\end{equation*}

Suppose $ \langle \theta, \theta^* \rangle = 0 $. This implies that $ A = 0 $. To see this, note that
\begin{equation} \label{eq:half}
\expect{\phi(W\langle \theta, X \rangle/\sigma^2)} = \expect{\phi(\Lambda Z_1|Z_2|)} = \phi(0) = 1/2,
\end{equation}
and
\begin{equation} \label{eq:zero}
\expect{W\langle \theta, X \rangle\phi^{\prime}(W\langle \theta, X \rangle/\sigma^2)} = \sigma^2\expect{\Lambda Z_1|Z_2|\phi^{\prime}(\Lambda Z_1|Z_2|)} = 0.
\end{equation}
The first equality \prettyref{eq:half} follows from the the fact that if $ Z \sim N(0, 1) $, then $ \expect{\phi(z Z)} = 1/2 $ for all $ z $ in $ \mathbb{R} $. This fact is easily established by noting that the derivative with respect to $ z $ is zero everywhere. The expectation in \prettyref{eq:zero} vanishes since we are averaging an odd function with respect to a symmetric distribution.
Next, observe that $ B = 2(1+\|\theta^*\|^2/\sigma^2)\expect{Z^2_2\phi^{\prime}(Z_1|Z_2|(\|\theta\|/\sigma^2)\sqrt{\sigma^2+\|\theta^*\|^2})} \rightarrow 1+\|\theta^*\|^2/\sigma^2 > 1 $ as $ \theta \rightarrow 0 $. By continuity, there exists $ r > 0 $ such that if $ \|\theta\| = r $, then $ B > 1 $, and hence
\begin{align*}
\|M(\theta) - \theta^*\|^2 & = \|\theta-\theta^*\|^2 + (B^2-1)\|\theta\|^2 \\ 
& > \|\theta-\theta^*\|^2.
\end{align*}
This shows that
\begin{equation*}
\liminf_{\langle \theta, \theta^* \rangle \downarrow 0,\; \|\theta\| = r}[\|M(\theta) - \theta^*\|^2-\|\theta-\theta^*\|^2] > 0.
\end{equation*}
By continuity, it follows that there exists $ r' > 0 $ such that if $ 0 < \langle \theta, \theta^* \rangle < r' $ then $ \|M(\theta) - \theta^*\|^2 >  \|\theta-\theta^*\|^2 $. It is easy to see that the set of all points satisfying $ 0 < \langle \theta, \theta^* \rangle < r' $ and $ 0 < \|\theta\| < r $ has positive Lebesgue measure and satisfies the stated conditions in the lemma.
\end{proof}

For the following lemmas, recall the definitions
\begin{equation*}
A = \expect{2\phi(W \langle \theta, X \rangle/\sigma^2) + 2(W\langle \theta, X \rangle/\sigma^2) \phi^{\prime}(W \langle \theta, X \rangle/\sigma^2)-1},
\end{equation*}
\begin{equation*}
B = 2\expect{(W^2/\sigma^2)\phi^{\prime}(W \langle \theta, X \rangle/\sigma^2)},
\end{equation*}
and
\begin{equation*}
\kappa^2 = \frac{1}{\frac{\Gamma}{\Lambda}\min\left\{\Lambda,\frac{\Gamma}{\Lambda}\right\}+1} = \max\left\{1-\frac{|\langle\theta_0,\theta^\star\rangle|^2}{\sigma^2+\|\theta^*\|^2}, 1-\frac{\langle\theta,\theta^*\rangle}{\sigma^2+\langle\theta,\theta^*\rangle} \right\}.
\end{equation*}

\begin{lemma} \label{lmm:Mbounds}
The cosine angle between $ \theta^* $ and $ M(\theta) $ is equal to
\begin{equation} \label{eq:cosinea}
\frac{\|\theta^*\|^2A + \langle \theta, \theta^* \rangle B}{\sqrt{(\|\theta^*\|^2A + \langle \theta, \theta^* \rangle B)^2 + B^2(\|\theta\|^2\|\theta^*\|^2 - |\langle \theta, \theta^* \rangle|^2)}}.
\end{equation}
If $ \langle \theta, \theta^* \rangle \geq \rho\|\theta\|\|\theta^*\| $, then there exists positive $ \Delta = \Delta(\rho, \sigma, \|\theta^*\|, \|\theta\|) $ such that the cosine angle \prettyref{eq:cosinea} is at least $ (1+\Delta)\rho $.
Moreover, if $ \langle \theta^*, \theta \rangle \geq 0 $, then
\begin{equation} \label{eq:cosinea1}
\|\theta^*\|^2(1-\kappa)^2 \leq \|M(\theta)\|^2 = \|\theta^*\|^2A^2 + \|\theta\|^2B^2+2\langle \theta, \theta^* \rangle AB \leq \sigma^2 + 3\|\theta^*\|^2,
\end{equation}
and
\begin{equation} \label{eq:cosinea2}
\langle \theta^*, M(\theta) \rangle = \|\theta^*\|^2A + \langle \theta, \theta^* \rangle B \geq \|\theta^*\|^2(1-\kappa).
\end{equation}
\end{lemma}
\begin{proof}
The stated expression \prettyref{eq:cosinea} for the cosine angle between $ \theta^* $ and $ M(\theta) $ comes from the expression $ \frac{\langle u, v\rangle}{\|u\|\|v\|} = \frac{\langle \theta^*, M(\theta)\rangle}{\|\theta^*\|\|M(\theta)\|} $ for the cosine angle between two vectors $ u $ and $ v $, and the fact that $ M(\theta) = A\theta^* + B\theta $  (see \prettyref{eq:span}).

Next, we prove the second statement about the lower bound on \prettyref{eq:cosinea}. Let $ \tau = \frac{\|\theta^*\|}{\|\theta\|}\frac{A}{B} $. Observe that
\begin{align}
& \frac{\|\theta^*\|^2A + \langle \theta, \theta^* \rangle B}{\sqrt{(\|\theta^*\|^2A + \langle \theta, \theta^* \rangle B)^2 + B^2(\|\theta\|^2\|\theta^*\|^2 - |\langle \theta, \theta^* \rangle|^2) }} \nonumber \\
& = \frac{1}{\sqrt{1 + \frac{\|\theta\|^2\|\theta^*\|^2 - |\langle \theta, \theta^* \rangle|^2}{(\|\theta^*\|^2\frac{A}{B} + \langle \theta, \theta^* \rangle)^2 } }} \nonumber \\
& \geq \frac{1}{\sqrt{1 + \frac{1-\rho^2}{(\tau+\rho)^2}}} \nonumber \\
& = \frac{\rho}{\sqrt{1 - (1-\rho^2)\frac{\tau(\tau+2\rho)}{(\tau+\rho)^2}}} \nonumber \\
& \geq \frac{\rho}{\sqrt{1 - (1-\rho^2)\frac{\tau}{\tau+\rho}}} \nonumber \\
& \geq \rho\left(1+\frac{1}{2}(1-\rho^2)\frac{\tau}{\tau+\rho}\right), \label{eq:finalcosinea}
\end{align}
where the last line \prettyref{eq:finalcosinea} follows from the inequality $ 1/\sqrt{1-z} \geq 1+z/2 $ for all $ z \in (0, 1) $. Next, note that from \prettyref{lmm:ABbounds},
\begin{equation*}
\frac{A}{B} \geq \frac{\sigma^2(1-\kappa)}{2(\sigma^2+\|\theta^*\|^2)\kappa^3}.
\end{equation*}
Thus, $ \tau \geq \tau' := \frac{\sigma^2\|\theta^*\|(1-\kappa)}{2\|\theta\|(\sigma^2+\|\theta^*\|^2)\kappa^3} $
and so we can set
\begin{equation*}
\Delta = \frac{1}{2}(1-\rho^2)\left(\frac{\tau'}{\tau'+\rho}\right) > 0.
\end{equation*}

For the statement in \prettyref{eq:cosinea1}, the identity 
\begin{equation*}
\|M(\theta)\|^2 = \|\theta^*\|^2A^2 + \|\theta\|^2B^2+2\langle \theta, \theta^* \rangle AB
\end{equation*}
is an immediate consequence of $ M(\theta) = A\theta^* + B\theta $. By \prettyref{lmm:ABbounds}, $ A \geq 1-\kappa $ and hence since $ \langle \theta, \theta^* \rangle \geq 0 $, we have $ \|M(\theta)\|^2 \geq \|\theta^*\|^2A^2 \geq \|\theta^*\|^2(1-\kappa)^2 $.

Next, we will show that $ \|M(\theta)\|^2 \leq \sigma^2+3\|\theta^*\|^2 $ for all $ \theta $ in $ \mathbb{R}^d $. To see this, note that by Jensen's inequality,
\begin{align*}
\langle \theta, M(\theta) \rangle & = \expect{(2\phi(W\langle \theta, X \rangle/\sigma^2)-1)W\langle \theta, X \rangle} \\ 
& \leq \expect{|W\langle \theta, X \rangle|} \\
& \leq \sqrt{\expect{|W\langle \theta, X \rangle|^2}} \\
& = \sigma^2\sqrt{\Lambda^2 + 3\Gamma^2} \\
& = \|\theta\|\sqrt{\sigma^2+\|\theta^*\|^2 + 2|\langle \theta_0, \theta^* \rangle|^2 }.
\end{align*}
Next, it can be shown that $ |2\phi(z)+2z\phi^{\prime}(z)-1| \leq \sqrt{2} $ and hence $ A \leq \sqrt{2} $. Using this, we have
\begin{align*}
\langle \theta^{\perp}_0, M(\theta) \rangle & =  A\langle \theta^{\perp}_0, \theta^* \rangle \\
& \leq \sqrt{2}\langle \theta^{\perp}_0, \theta^* \rangle.
\end{align*}
Putting these two facts together, we have
\begin{align*} \label{eq:Mbound}
\|M(\theta)\|^2 & = |\langle \theta^{\perp}_0, M(\theta) \rangle|^2 + |\langle \theta_0, M(\theta) \rangle|^2 \\
& \leq \sigma^2 + \|\theta^*\|^2 + 2|\langle \theta^{\perp}_0, \theta^* \rangle|^2 + 2|\langle \theta_0, \theta^* \rangle|^2 \\
& = \sigma^2 + 3\|\theta^*\|^2.
\end{align*}

The final statement \prettyref{eq:cosinea2} follows from similar arguments and so we omit them here.
\end{proof}

\begin{lemma} \label{lmm:positive}
If $ \langle \theta, \theta^* \rangle \geq 0 $, then
\begin{equation*}
\expect{W\langle \theta, X \rangle \phi^{\prime}(W \langle \theta, X \rangle/\sigma^2)} \geq 0.
\end{equation*}
\end{lemma}
\begin{proof}
Writing $ W\langle \theta, X \rangle $ according to the distributional equivalent \prettyref{eq:dist2}, note that the statement is true if
\begin{equation*}
\expect{(\alpha Z + \beta)\phi^{\prime}(\alpha Z + \beta)} \geq 0,
\end{equation*}
where $ Z \sim N(0, 1) $ and $ \alpha \geq 0 $ and $ \beta \geq 0 $.
This fact is proved in \cite[Lemma 5]{Klusowski2016} or \cite[Lemma 1]{daskalakis2016ten}.
\end{proof}

\begin{lemma} \label{lmm:inequality}
The following inequalities hold for all $ z \in \mathbb{R} $:
\begin{equation*}
|2\phi(z)+2z\phi^{\prime}(z)-1| \leq 1+\sqrt{2(1-\phi(z))},
\end{equation*}
and
\begin{equation*}
z^2\phi^{\prime}(z) \leq \sqrt{2(1-\phi(z))}.
\end{equation*}
\end{lemma}
\begin{proof}
Their validity can easily be established using mathematical software.
\end{proof}

\begin{lemma} \label{lmm:main_inequality}
Let $ \alpha, \beta > 0 $ and $ Z \sim N(0,1) $. Then
\begin{equation*}
\expect{2(1-\phi(\alpha(Z + \beta)))} \leq \exp\left\{-\frac{\beta}{2}\min\{\alpha, \beta\}\right\}.
\end{equation*}
Moreover,
\begin{equation*}
\expect{2(1-\phi(\alpha |Z_2|(Z_1 + \beta |Z_2|)))} \leq \frac{1}{\sqrt{\beta\min\{\alpha,\beta\}+1}}.
\end{equation*}
\end{lemma}
\begin{proof}
The second conclusion follows immediately from the first since
\begin{align*}
\expect{2(1-\phi(\alpha |Z_2|(Z_1 + \beta |Z_2|)))}
& = 2\Expect_{Z_2}\left[\Expect_{Z_1}\left[1-\phi(\alpha |Z_2|(Z_1 + \beta |Z_2|))\right]\right] \\
& \leq \Expect_{Z_2}\left[\exp\left\{-\frac{Z^2_2}{2}\beta\min\{\alpha,\beta\}\right\}\right] \\
& = \frac{1}{\sqrt{\beta\min\{\alpha,\beta\}+1}}.
\end{align*}
The last equality follows from the moment generating function of $\chi^2_1$.

For the first conclusion, we first observe that the mapping $ \alpha \mapsto \expect{\phi(\alpha(Z + \beta))} $ is increasing (see \cite[Lemma 5]{Klusowski2016} or \cite[Lemma 1]{daskalakis2016ten}). Next, note the inequality
\begin{equation*}
2(1-\phi(z)) \leq e^{-z},
\end{equation*}
which is equivalent to $ (e^z-1)^2 \geq 0 $.
If $ \alpha \geq \beta $, then
\begin{align*}
\expect{2(1-\phi(\alpha(Z + \beta)))}
& \leq \expect{2(1-\phi(\beta(Z + \beta)))} \\
& \leq \expect{e^{-(\beta(Z + \beta))}} \\
& = e^{-\beta^2/2}.
\end{align*}
If $ \alpha \leq \beta $, then
\begin{align*}
\expect{2(1-\phi(\alpha(Z + \beta)))}
& \leq \expect{e^{-(\alpha(Z + \beta))}} \\
& = e^{\alpha^2/2-\alpha\beta} \\
& \leq e^{-\alpha\beta/2}.
\end{align*}
In each case, we used the moment generating function of a Gaussian distribution to evaluate the expectations.
\end{proof}

\begin{lemma} \label{lmm:ABbounds}
The following inequalities hold:
\begin{equation*}
1-\kappa \leq A \leq 1+\sqrt{\kappa},
\end{equation*}
and
\begin{equation*}
B \leq 2(1+\|\theta^*\|^2/\sigma^2)\kappa^3.
\end{equation*}
\end{lemma}
\begin{proof}
By \prettyref{lmm:positive} and \prettyref{lmm:main_inequality},
\begin{align*}
A
& = \expect{2\phi(W \langle \theta, X \rangle/\sigma^2) + 2(W\langle \theta, X \rangle/\sigma^2) \phi^{\prime}(W \langle \theta, X \rangle/\sigma^2)-1} \\ & \geq \expect{2\phi(W \langle \theta, X \rangle/\sigma^2)-1} \\
& \geq 1-\kappa.
\end{align*}
By \prettyref{lmm:inequality}, Jensen's inequality, and \prettyref{lmm:main_inequality},
\begin{align*}
A
& = \expect{2\phi(W \langle \theta, X \rangle/\sigma^2) + 2(W\langle \theta, X \rangle/\sigma^2) \phi^{\prime}(W \langle \theta, X \rangle/\sigma^2)-1} \\ & \leq \expect{1+\sqrt{2(1-\phi(W \langle \theta, X \rangle/\sigma^2))}} \\
& \leq 1+\sqrt{\expect{2(1-\phi(W \langle \theta, X \rangle/\sigma^2))}} \\
& \leq 1+\sqrt{\kappa}.
\end{align*}
By the inequality $ \phi'(z) \leq 2(1-\phi(z)) $ for all $ z \in \mathbb{R} $ and \prettyref{lmm:main_inequality},
\begin{align*}
B
& = 2\expect{(W^2/\sigma^2)\phi^{\prime}(W \langle \theta, X \rangle/\sigma^2)} \\
& \leq 2\expect{2(W^2/\sigma^2)(1-\phi(W \langle \theta, X \rangle/\sigma^2))} \\
& = 2(1+\|\theta^*\|^2/\sigma^2)\Expect_{Z_2}\left[Z^2_2\Expect_{Z_1}\left[2\left(1-\phi\left(\Lambda |Z_2|\left(Z_1+\frac{\Gamma}{\Lambda}|Z_2|\right)\right)\right)\right]\right] \\
& \leq 2(1+\|\theta^*\|^2/\sigma^2)\Expect_{Z_2}\left[Z^2_2\exp\left\{ -\frac{Z^2_2}{2}\frac{\Gamma}{\Lambda}\min\left\{\frac{\Gamma}{\Lambda},\Lambda\right\} \right\}\right] \\
& = 2(1+\|\theta^*\|^2/\sigma^2)\left(\frac{1}{\frac{\Gamma}{\Lambda}\min\left\{\Lambda,\frac{\Gamma}{\Lambda}\right\}+1}\right)^{3/2} \\
& = 2(1+\|\theta^*\|^2/\sigma^2)\kappa^3. \qedhere
\end{align*}
\end{proof}

\begin{lemma} \label{lmm:derivative}
Define
\begin{equation*}
h(\alpha, \beta) = \expect{(2\phi(\alpha |Z_2|(Z_1 + \beta |Z_2|))-1)(|Z_2|(Z_1 + \beta |Z_2|))}.
\end{equation*} 
Let $ \alpha, \beta > 0 $. Then
\begin{equation*}
\frac{\partial}{\partial \alpha} h(\alpha, \beta) \leq \frac{2}{\alpha^2}\left(\frac{1}{\beta\min\{\alpha,\beta\}+1}\right)^{1/4}.
\end{equation*}
\end{lemma}
\begin{proof}
First, observe that
\begin{align*}
\frac{\partial}{\partial \alpha} h(\alpha, \beta)
& = \expect{2\phi^{\prime}(\alpha |Z_2|(Z_1 + \beta |Z_2|))(|Z_2|(Z_1 + \beta |Z_2|))^2}.
\end{align*}
By \prettyref{lmm:inequality}, Jensen's inequality, and \prettyref{lmm:main_inequality},
\begin{align*}
& \expect{2\phi^{\prime}(\alpha |Z_2|(Z_1 + \beta |Z_2|))(|Z_2|(Z_1 + \beta |Z_2|))^2} \\
& = \frac{1}{\alpha^2}\expect{2\phi^{\prime}(\alpha |Z_2|(Z_1 + \beta |Z_2|))(\alpha |Z_2|(Z_1 + \beta |Z_2|))^2} \\
& \leq \frac{2}{\alpha^2}\expect{\sqrt{2(1-\phi(\alpha |Z_2|(Z_1 + \beta |Z_2|)))}} \\
& \leq \frac{2}{\alpha^2}\sqrt{\expect{2(1-\phi(\alpha |Z_2|(Z_1 + \beta |Z_2|)))}} \\
& \leq \frac{2}{\alpha^2}\left(\frac{1}{\beta\min\{\alpha,\beta\}+1}\right)^{1/4}. \qedhere
\end{align*}
\end{proof}

\bibliographystyle{plain}
\bibliography{MixtureRegressionEMReferences}

\end{document}